\def\year{2021}\relax
\documentclass[hidelinks,letterpaper]{article} 
\usepackage{aaai21}  
\usepackage{times}  
\usepackage{helvet} 
\usepackage{courier}  
\usepackage[hyphens]{url}  
\usepackage{graphicx} 
\usepackage{natbib}  
\usepackage{caption} 
\usepackage{mathtools}
\usepackage{subcaption}
\usepackage{amsthm}
\usepackage{thmtools} 
\usepackage{algorithmicx}
\usepackage{algorithm,algpseudocode}
 \usepackage{url}

\title{Deep Radial-Basis
Value Functions for Continuous Control}
\author{Kavosh Asadi\textsuperscript{\rm 1,2}, Neev Parikh\textsuperscript{\rm 2}, Ronald E. Parr\textsuperscript{\rm 3}, George D. Konidaris\textsuperscript{\rm 2}, Michael L. Littman\textsuperscript{\rm 2}
\\}
\affiliations{\textsuperscript{\rm 1} Amazon Web Services \\
\textsuperscript{\rm 2} Brown University \\
\textsuperscript{\rm 3} Duke University}
\begin{document}

\maketitle

\newcommand{\Qhat}{\widehat Q}
\newcommand{\fancyS}{\mathcal{S}}
\newcommand{\fancyA}{\mathcal{A}}
\newcommand{\fancyR}{\mathcal{R}}
\newcommand{\fancyO}{\mathcal{O}}
\newcommand\myeq{\stackrel{\mathclap{\normalfont\mbox{def}}}{=}}
\newcommand{\abs}[1]{|#1|}
\newcommand{\expectation}[1]{\mathbf{E}_{#1}}
\newcommand\norm[1]{\left\lVert#1\right\rVert}
\newtheorem{theorem}{Theorem}

\begin{abstract}
A core operation in reinforcement learning (RL) is finding an action that is optimal with respect to a learned value function. This operation is often challenging when the learned value function takes continuous actions as input. We introduce deep radial-basis value functions (RBVFs): value functions learned using a deep network with a radial-basis function (RBF) output layer. We show that the maximum action-value with respect to a deep RBVF can be approximated easily and accurately. Moreover, deep RBVFs can represent any true value function owing to their support for universal function approximation. We extend the standard DQN algorithm to continuous control by endowing the agent with a deep RBVF. We show that the resultant agent, called RBF-DQN, significantly outperforms value-function-only baselines, and is competitive with state-of-the-art actor-critic algorithms.
\end{abstract}
\section{Introduction}
A fundamental object of interest in RL is the value ($Q$) function, which quantifies the expected return for taking  action $a$ in state $s$. Many RL algorithms, such as Q-learning~\citep{Q_learning} and actor-critic~\citep{barto1983neuronlike}, learn an approximation of the $Q$ function from environmental interactions. When using function approximation to learn the $Q$ function, the agent has a parameterized function class, and learning consists of finding a parameter setting $\theta$ for the approximate value function $\Qhat(s,a;\theta)$ that accurately represents the true $Q$ function. 
A core operation here is finding an optimal action with respect to the learned value function, specifically $\arg \max_{a\in\fancyA} \Qhat(s,a;\theta)$. The need for performing this operation arises not just for action selection, but also when learning $\Qhat$ itself via bootstrapping \cite{RL_book}.

The optimization problem $\arg \max_{a \in \fancyA} \widehat Q(s,a;\theta)$ is generally challenging if the action space $\fancyA$ is continuous, because the learned value function $\Qhat(s,a;\theta)$ could have many local maxima and saddle points; therefore, na\"ive approaches such as gradient ascent can be expensive and inaccurate~\citep{actor_expert}. Value-function-only algorithms have thus remained under-explored in continuous control.

It is trivial to solve $\arg \max_{a \in \fancyA} \widehat Q(s,a;\theta)$ when  $\fancyA$ is discrete. Therefore, a simple approach for continuous actions is to partition the continuous action space. By performing this discretization, we can essentially treat the original continuous problem as a discrete one. While this approach can be effective in highly constrained settings, to uniformly cover the action space, we need a partition size that grows exponentially with the number of action dimensions. \citet{pazis2011generalized} attempt to combat this issue via approximate linear programming, but scaling this approach to large domains remains open. 

Another line of work has shown the benefits of using function classes that are conducive to efficient action maximization. For example, \citet{naf} explored function classes that can capture an arbitrary dependence on the state, but only a quadratic dependence on the action. Given a quadratic action dependence, \citet{naf} showed how to compute $\arg \max_{a\in\fancyA} \Qhat(s,a;\theta)$ quickly. A more general idea is to use input--convex neural networks~\citep{icnn} that restrict $\Qhat(s,a;\theta)$ to convex (or concave) functions with respect to $a$, so that the maximization problem can be solved efficiently using convex-optimization techniques~\citep{convex}. Unfortunately, these approaches are unable to support universal function approximation~\citep{ufa_feed_forward}, and may lead into inaccurate value functions regardless of the amount of experience provided to the agent.

We introduce deep radial-basis value functions (RBVFs): $Q$ functions approximated by a standard deep neural network augmented with an RBF output layer.  We show that deep RBVFs enable us to efficiently and accurately identify an approximately optimal action without impeding universal function approximation. 

Prior work in RL used RBVFs for problems with continuous states. (See, for example, \citet{geramifard_tutorial} or Section 9.5.5 of \citet{RL_book} for a discussion.) These results often use RBFs to transform the original state space to a higher-dimensional space in which linear function approximation yields accurate results. Our use of RBFs is fundamentally different in two ways: a) we apply RBFs to the action space not the state space, and b) we use deep neural networks to learn the centroids that constitute the RBFs. Applying RBFs to the action space allows us to show that the global maximum of $\Qhat(s,\cdot;\theta)$ is approximately equal to the value of the best centroid---$\max_{a} \Qhat(s,a;\theta) \approx \max_{i\in[1,N]} \Qhat(s,a_i;\theta)$ where $a_i$ represents the $i$th centroid. Because computing $\max_{i\in[1,N]} \Qhat(s,a_i;\theta)$ is easy in an RBVF, this result provides tremendous leverage.

The aim of our experimental results are twofold. We first endow DQN, a standard deep RL algorithm originally proposed for discrete actions~\citep{DQN}, with a deep RBVF and produce a new continuous-control algorithm called RBF-DQN. We evaluate RBF-DQN to demonstrate its superior performance relative to value-function-only baselines, and to show its competitiveness with state-of-the-art actor-critic deep RL. We also show that a deep RBVF could serve as the critic in standard actor-critic algorithms such as DDPG~\citep{dpg,ddpg}. This enables us to approximate the greedy value function via the critic, and ultimately improve the performance of DDPG.
\section{Background}
We study the interaction between an environment and an agent that seeks to maximize reward~\citep{RL_book}, a problem typically formulated using Markov Decision Processes (MDPs)~\citep{puterman_mdp}. An MDP is specified by a tuple: $\langle \fancyS, \fancyA, T, R, \gamma \rangle$. In this work, $\fancyS$ and $\fancyA$ denote the continuous state space and the continuous action space of the MDP, and we further assume that they are closed and bounded, therefore compact. The goal of an RL agent is to find a policy $\pi:\fancyS\rightarrow\fancyA$ that collects high sums of discounted rewards across timesteps.

For a state $s\in\fancyS$, action $a\in\fancyA$, and a policy $\pi$, we define the state--action value function:
$$Q^{\pi}(s,a):= \expectation{\pi}\big[G_t\mid s_t=s,a_t=a\big]\ , $$
where $G_t:=\sum_{i=t}^{\infty}\gamma^{i-t}R_{i}$ is called the \emph{return} at timestep $t$. The state--action value function of an \textit{optimal} policy, denoted by $Q^{*}(s,a)$, can be written recursively \citep{bellman}:
\begin{equation*}
    Q^{*}(s,a)=R(s,a)+\gamma\int_{s'}\!T(s,a,s')\max_{a'}Q^{*}(s',a')ds'.
    \label{eq:Bellman}
\end{equation*}

In the absence of a model, a class of RL algorithms solve for the fixed point of Bellman's equation using environmental interactions. Q-learning \citep{Q_learning}, a notable example of these so-called model-free algorithms, learns an approximation of $Q^*$, denoted by $\Qhat$ and parameterized by $\theta$. Under function approximation, $\theta$ parameters are updated as follows:
\begin{equation}
    \begin{split}
        &\theta\leftarrow\theta+\alpha\ \delta \ \nabla_\theta \Qhat(s,a;\theta)\ ,\\
        & \textrm{where}\quad \delta := r+\gamma \max_{a'\in\fancyA} \Qhat(s',a';\theta)-\Qhat(s,a;\theta)\ ,
    \end{split}
    \label{eq:q_learning}
\end{equation}
using tuples of experience $\langle s,a,r,s'\rangle$ collected during environmental interactions. Note that Q-learning's update rule (\ref{eq:q_learning}) is agnostic to the choice of function class, and so in principle any differentiable and parameterized function class could be used in conjunction with the above update to learn $\theta$ parameters. For example, \citet{sutton_linear} used linear function approximation, \citet{konidaris_fourier} used Fourier basis functions, and \citet{DQN} used convolutional neural networks. Finally, we assume that $\Qhat$ is continuous.

\section{Deep Radial-Basis Value Functions}
\label{sec:Deep_RBF_Value_Functions}
Deep Radial-Basis Value Functions (RBVFs) combine the practical advantages of deep networks \citep{deep_learning_book} with the  theoretical advantages of radial-basis functions (RBFs) \citep{rbf_survey}. A deep RBVF is comprised of a number of arbitrary hidden layers, followed by an \textit{RBF output layer}, defined next. The RBF output layer, first introduced in a seminal paper by \citet{rbf}, is sometimes used as a standalone single-layer function approximator, referred to as a (shallow) RBF network. It is also a core ingredient of the kernel trick in Support Vector Machines \citep{cortes1995support}. We use an RBF network as the final, or output, layer of a deep network.

For a given input $a$, the RBF layer $f(a)$ is defined as:
\begin{equation}
    f(a):=\sum_{i=1}^{N} g(a-a_i)\ v_i \ , \label{eq:RBF_interpolation}
\end{equation}
where each $a_i$ represents a \emph{centroid} location, $v_i$ is the value of the centroid $a_i$, $N$ is the number of centroids, and $g$ is an RBF. A commonly used RBF is the negative exponential:
\begin{equation}
    g(a-a_i):=e^{-\beta\norm{a-a_i}}\ , \label{eq:negative_exponential_RBF}
\end{equation}
equipped with an inverse smoothing parameter $\beta\!\geq\!0$. (See \citet{reformulated_rbf} for a thorough treatment of other RBFs.) Formulation (\ref{eq:RBF_interpolation}) could be thought of as an interpolation based on the value and the weights of all centroids, where the weight of each centroid is determined by its proximity to the input. Proximity here is quantified by the RBF $g$, in this case the negative exponential (\ref{eq:negative_exponential_RBF}).

It is theoretically useful to normalize centroid weights to ensure that they sum to 1 so that $f$ implements a weighted average. This weighted average is sometimes referred to as a normalized RBF layer~\citep{moody1989fast,normalized_RBF}:
\begin{equation}
    f_{\beta}(a):=\frac{\sum_{i=1}^{N}e^{-\beta\norm{a-a_i}}\ v_i}{\sum_{i=1}^{N}e^{-\beta\norm{a-a_i}}}\ .
    \label{normalized_RBF_formulation}
\end{equation}
As the inverse smoothing parameter $\beta\!\rightarrow\!\infty$, the function implements a winner-take-all case where the value of the function at a given input is determined only by the value of the closest centroid location, nearest-neighbor style. This limiting case is sometimes referred to as a \emph{Voronoi decomposition} \citep{voronoi}. Conversely, $f$ converges to the mean of centroid values regardless of the input $a$ as $\beta$ gets close to 0; that is, $\forall a \ \lim_{\beta\rightarrow 0} f_{\beta}(a)~=~\frac{\sum_{i=1}^{N} v_i}{N}$. Since an RBF layer is differentiable, it could be used in conjunction with gradient-based optimization techniques and Backprop to learn the centroid locations and their values by optimizing for a loss function. Note that formulation (\ref{normalized_RBF_formulation}) is different than the Boltzmann softmax operator \citep{mellowmax,song2019revisiting}, where the weights are determined not by an RBF, but by the action values.

Finally, to represent the $Q$ function for RL, we use the following formulation:
\begin{equation}
    \Qhat_{\beta}(s,a;\theta):=\frac{\sum_{i=1}^{N}e^{-\beta\norm{a-a_i(s;\theta)}}\ v_i(s;\theta)}{\sum_{i=1}^{N}e^{-\beta\norm{a-a_i(s;\theta)}}}\ .
    \label{eq:deep_RBF_value_function}
\end{equation}
From equation (\ref{eq:deep_RBF_value_function})
a deep RBVF learns two mappings: state-dependent centroid locations $a_i(s;\theta)$ and state-dependent centroid values $v_i(s;\theta)$. The role of the output layer is to compute the output of the entire deep RBVF.  We illustrate the architecture of a deep RBVF in Figure~\ref{fig:RBF_network}. In the experimental section, we demonstrate how to learn parameters $\theta$.

\begin{figure}
    \centering
    \includegraphics[width=\linewidth]{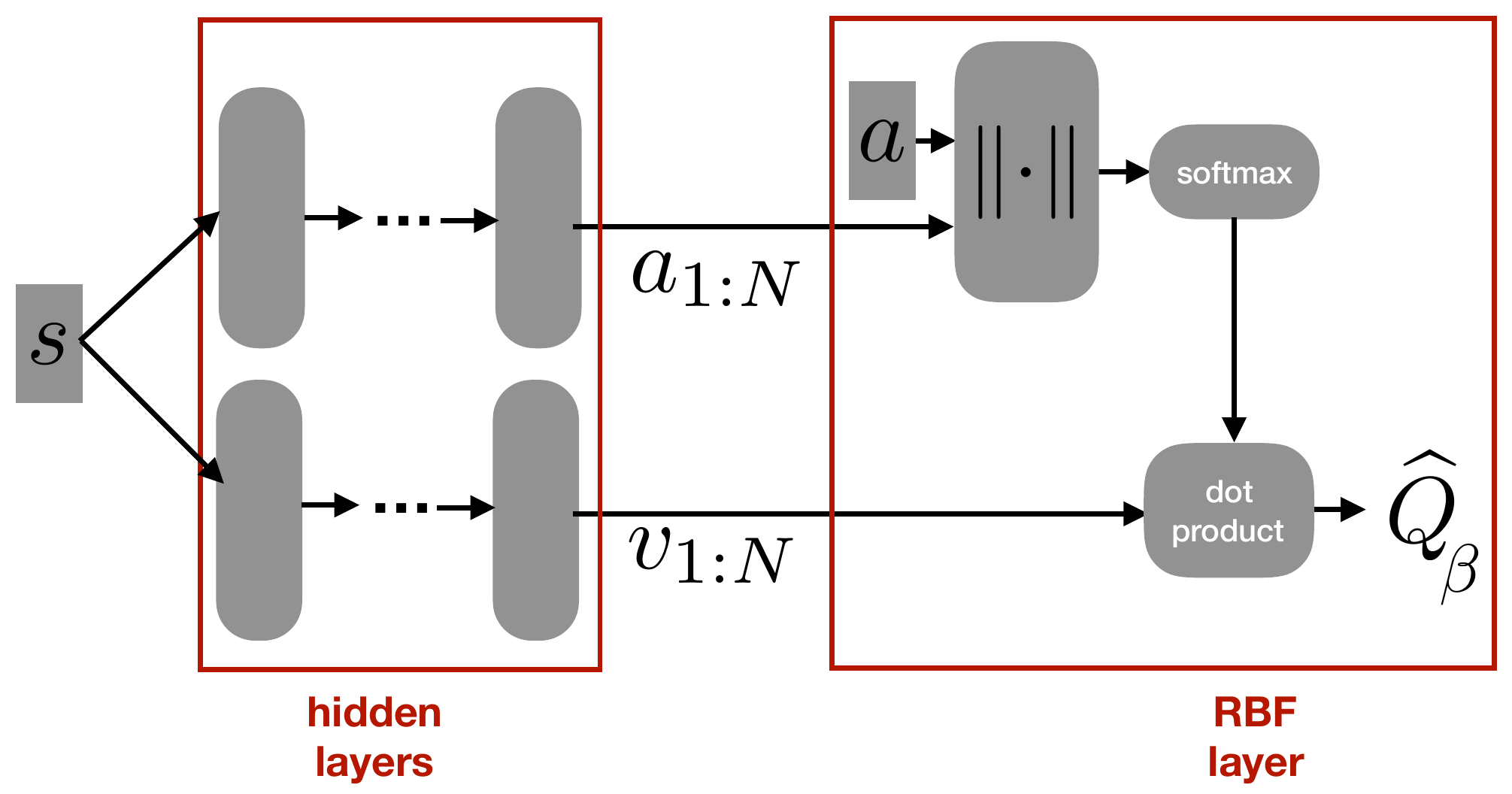}
    \caption{Architecture of a deep RBVF, which could be thought of as an RBF output layer added to an otherwise standard deep $Q$ function. All operations of the final RBF layer are differentiable,  so the parameters of hidden layers $\theta$, which represent the mappings $a_i(s;\theta)$ and $v_i(s;\theta)$, can be learned using gradient-based optimization techniques.}
    \label{fig:RBF_network}
\end{figure}
We now show that deep RBVFs have a highly desirable property for value-function-based RL, namely that they enable easy action maximization. 

First, it is easy to find the output of a deep RBVF at each centroid location $a_i$, that is, to compute $\Qhat_{\beta}(s,a_i;\theta)$.
Note that $\Qhat_{\beta}(s,a_i;\theta)\neq v_i(s;\theta)$ in general for a finite $\beta$, because the other centroids $a_j\ \forall j \in \{1,..,N\}-{i}$ may have non-zero weights at $a_i$. To compute $\Qhat_{\beta}(s,a_i;\theta)$, we access the centroid location using $a_i(s;\theta)$, then input $a_i$ to get $\Qhat(s,a_i;\theta)\ $. Once we have $\Qhat(s,a_i;\theta)\ \forall a_i$, we can trivially find the best centroid :$\ \max_{i\in[1,N]} \Qhat_{\beta}(s,a_i;\theta) \ .$

Recall that our goal is to compute $\max_{a\in\fancyA} \Qhat_{\beta}(s,a;\theta)$, but so far we have shown how to compute $\max_{i\in[1,N]} \Qhat_{\beta}(s,a_i;\theta)$. We now show that these two quantities are equivalent in one-dimensional action spaces. More importantly, Theorem~\ref{theorem:maxGap} shows that with arbitrary number of dimensions, there may be a gap, but that this gap gets exponentially small with increasing the inverse smoothing parameter $\beta$. Proofs are in the Appendix.
\begin{restatable}{theorem}{maxGap}
Let $\Qhat_{\beta}$ be a normalized negative-exponential RBVF.
\begin{enumerate}
    \item{ For $\mathcal{A} \in \mathcal{R}:$\\
    $\quad \max_{a\in\fancyA} \Qhat_{\beta}(s,a;\theta)=\max_{i\in[1,N]} \Qhat_{\beta}(s,a_i;\theta)\ .$}
    \item{ For $\fancyA \in \fancyR^d\quad \forall d>1 :$\\
    $\max_{a\in\fancyA}\! \Qhat_{\beta}(s,a;\theta)-\!\max_{i\in[1,N]}\!\Qhat_{\beta}(s,a_i;\theta)\!\leq\!\mathcal{O}(e^{-\beta}) \ .$}
\end{enumerate}
\label{theorem:maxGap}
\end{restatable}
\begin{figure}
\begin{subfigure}{.2\textwidth}
  \centering
  \includegraphics[width=.975\linewidth]{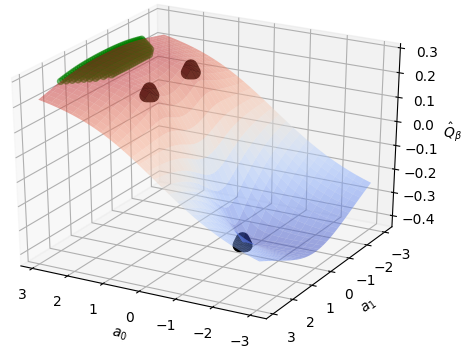}
  \caption{$\beta=0.25$}
\end{subfigure}
\begin{subfigure}{.2\textwidth}
  \centering
  \includegraphics[width=.975\linewidth]{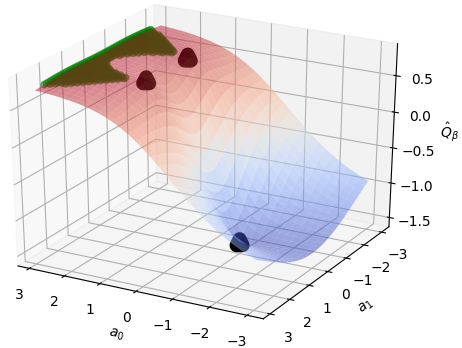}
  \caption{$\beta=1$}
\end{subfigure}
\begin{subfigure}{.2\textwidth}
  \centering
  \includegraphics[width=.975\linewidth]{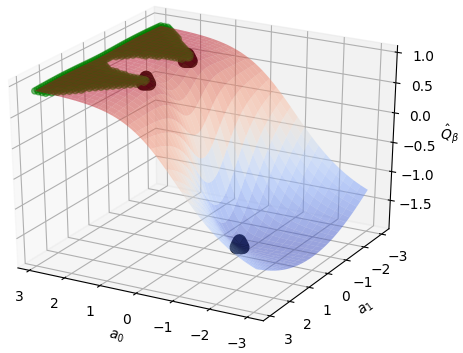}
  \caption{$\beta=1.5$}
\end{subfigure}\hspace{12mm}
\begin{subfigure}{.2\textwidth}
  \centering
  \includegraphics[width=.975\linewidth]{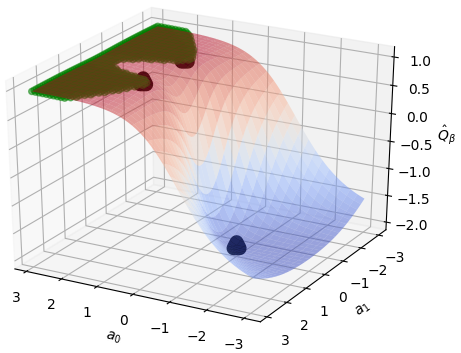}
  \caption{$\beta=2$}
\end{subfigure}
\caption{The output of a RBVF with 3 fixed centroid locations and values, but different settings of the inverse smoothing parameter $\beta$ on a 2-dimensional action space. The regions in dark green highlight the set of actions $a$ for which $a$ is extremely close to the global maximum, $\max_{a\in\fancyA} \Qhat_{\beta}(s,a;\theta)$. Observe the reduction of the gap between $\max_{a\in\fancyA} \Qhat_{\beta}(s,a;\theta)$ and $\max_{i\in[1,N]} \Qhat_{\beta}(s,a_i;\theta)$ by increasing $\beta$, as guaranteed by Theorem~\ref{theorem:maxGap}.}
\label{fig:effect_of_beta}
\end{figure}

In light of Theorem~\ref{theorem:maxGap}, to approximate $\max_{a\in\fancyA} \Qhat(s,a;\theta)$ we simply compute $\max_{i\in[1,N]} \Qhat(s,a_i;\theta)$. The accuracy of the approximation hinges on the magnitude of the inverse smoothing parameter. Notice that using extremely high values of $\beta$ may lead to poor generalization error due to the complexity of the resultant function class. Therefore, the theorem should be used as a guarantee on the approximation error, not as a justification for using extremely large $\beta$ values. Notice also that this result holds for normalized negative-exponential RBVFs, but not necessarily for the unnormalized case or for RBFs other than negative exponential.

In Theorem~2, presented in the Appendix, we show that RBVFs support universal function approximation. Collectively, Theorems~\ref{theorem:maxGap} and~2 guarantee that deep RBVFs ensure accurate and efficient action maximization without impeding universal function approximation. This combination of properties stands in contrast with prior work that used function classes that enable easy action maximization but lack UFA~\citep{naf,icnn}, as well as prior work that preserved the UFA property but did not guarantee arbitrarily high accuracy when performing the maximization step \citep{actor_expert,caql}.

In terms of scalability, note that the RBVF formulation scales naturally owing to its freedom to determine centroids that best minimize the loss function. As a thought experiment, suppose that some region of the action space has a high value, so an agent with greedy action selection frequently chooses actions from that region. A deep RBVF would then move more centroids to the region, because the region heavily contributes to the loss function. Since the centroids are state-dependent, the network can learn to move the centroids to more rewarding regions on a per-state basis. It is unnecessary to initialize centorid locations carefully, or to uniformly cover the action space \emph{a priori}. In this sense, learning RBVFs could be thought of as a form of adaptive and soft discretization learned by gradient descent. Adaptive discretization techniques have proven fruitful in terms of sample-complexity guarantees in bandit problems~\cite{zooming}, as well as in RL~\cite{zoom_rl0,zoom_rl,kernel_based_RL_regret}.
\section{Experiments}
In this section, we empirically assess the effectiveness of RBVFs in the context of non-linear regression, value-function-only deep RL, and actor-critic deep RL.
\subsection{RBVFs for Regression}
To demonstrate the operation of an RBF network in the clearest setting, we start with a regression task where the agent is provided with sample input--output pairs $\langle a,r\rangle$. We use the function: $r(a)=\norm{a}_2\frac{\sin(a_0)+\sin(a_1)}{2}$, whose surface is presented in Figure~\ref{fig:xsinx_problem} (left). It is non-convex and includes several local maxima (and minima). We sampled 500 inputs $a$ uniformly randomly from $[-3,3]^2$ to obtain a dataset for training. We used the mean-squared-error loss, and used $N=20$. The surface is learned well (Figure~\ref{fig:xsinx_problem}). 
\begin{figure}
  \centering
\begin{subfigure}{.25\textwidth}
  \includegraphics[width=.975\linewidth]{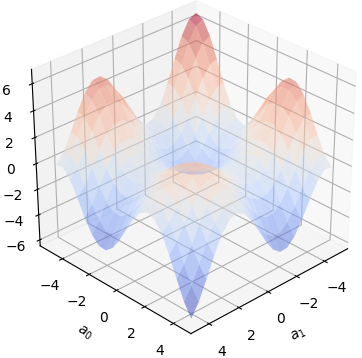}
\end{subfigure}%
\begin{subfigure}{.25\textwidth}
  \centering
  \includegraphics[width=.975\linewidth]{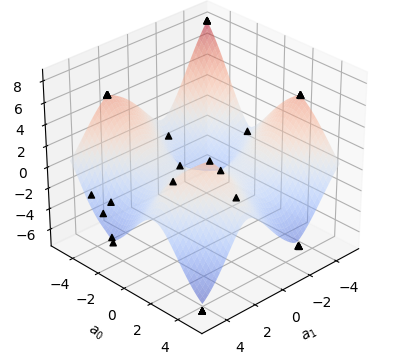}
\end{subfigure}
\caption{Left: True surface of the function. Right: The accurate approximation learned by the RBVF. Black dots represent the centroids. Note also that $\max_{i\in[1,N]} \widehat r_{\beta}(a_i;\theta)$ is an accurate approximation for $\max_{a\in\fancyA} r(a;\theta)$.}
\label{fig:xsinx_problem}
\end{figure}
\subsection{RBVFs for Value-Function-Only Deep RL}
We now use deep RBVFs for solving continuous-action RL problems. To this end, we learn a deep RBVF using a learning algorithm akin to that of DQN \citep{DQN}, but extended to the continuous-action case. DQN uses the following loss function for learning the value function: 
$$
    L(\theta):=\expectation{s,a,r,s'}\Big[\big(r+\gamma\max_{a'\in\fancyA}\Qhat(s',a';\theta^-)-\Qhat(s,a;\theta)\big)^2\Big].
    \label{eq:dqn_loss_function}$$
DQN adds tuples of experience $\langle s,a,r,s'\rangle$ to a buffer, and later samples a minibatch of tuples to compute $\nabla_{\theta} L(\theta)$. DQN maintains a second network parameterized by weights $\theta^{-}$. This second network, denoted $\Qhat(\cdot,\cdot,\theta^-)$ and referred to as the \emph{target network}, is periodically synchronized with the online network $\Qhat(\cdot,\cdot,\theta)$. RBF-DQN uses the same loss function, but modifies the function class of DQN. Concretely, DQN learns a deep network with one output per action, exploiting the discrete and finite nature of the action space.  By contrast, RBF-DQN takes a state and an action vector, and outputs a single scalar using a deep RBVF. The pseudo-code for RBF-DQN is presented in Algorithm 1.

\begin{figure*}[t]
  \centering
  \includegraphics[width=1\textwidth]{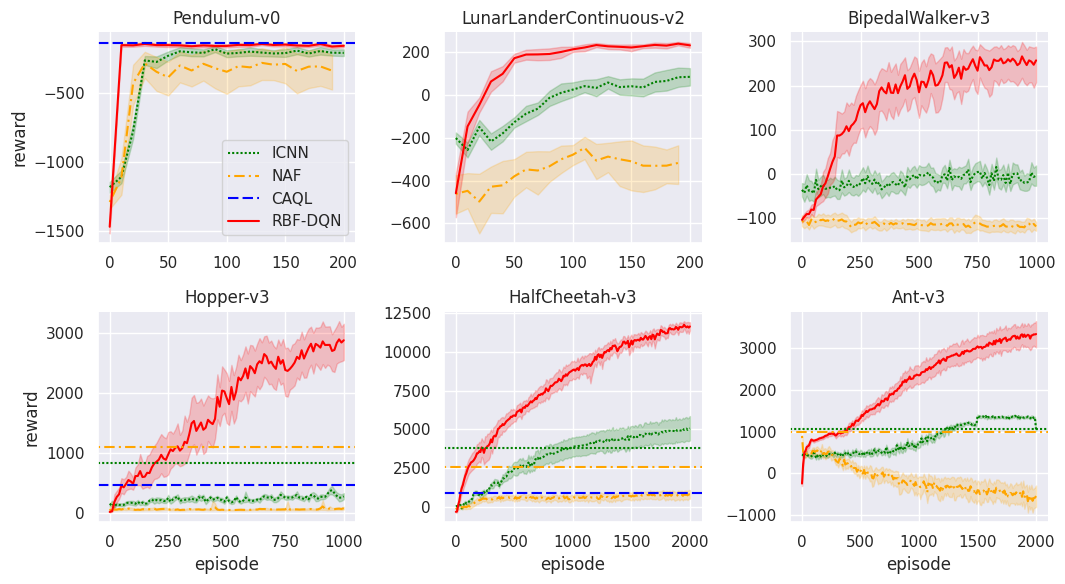}
    \caption{A comparison between RBF-DQN and value-function-only deep RL baselines on 6 Open AI Gym environments. For ICNN and NAF, we have included two results: Dashed lines indicate the level of performance reported by \citet{icnn}, and the learning curves show the performance of the baselines obtained by running publicly-available implementations. All runs for this figure and all following figures are averaged over 20 trials with different random seeds.}
  \label{fig:RL_results}
\end{figure*}

\begin{algorithm}
\begin{algorithmic}
\State Initialize deep RBVF with $N, \beta, \theta$
\State Initialize replay buffer $\mathcal{B}, \epsilon, \gamma,\alpha,\alpha^{-},\theta^{-}$
\For{$E\ \textrm{episodes}$}
\State Initialize $s$
    \While{\textrm{not done}}
    \State $a\leftarrow\epsilon\textrm{-greedy}\big(\Qhat_{\beta}(s,\cdot;\theta),\epsilon\big)$
    \State $s',r,\textrm{done}\leftarrow\textrm{env.step}(s,a)$
    \State add $\langle s,a,r,s',\textrm{done} \rangle$ to $\mathcal{B}$; $s\leftarrow s'$
    \EndWhile
    \For{$M$ minibatches sampled from $\mathcal{B}$}
        \For {$\langle s,a,r,s',\textrm{done}\rangle$ in minibatch}
        \State $\Delta=\big(r-\Qhat_{\beta}(s,a;\theta)\big)\nabla_{\theta}\Qhat(s,a;\theta)$
        \If{not done}
        \State \textrm{get centroids} $a_i(s';\theta^-), i \in [1,N]$
        \State $\Delta\!\mathrel{+}=\!\gamma\max_{i}\Qhat_{\beta}(s',a_i\big(s';\theta^-);\theta^-\big)$
        \EndIf
        \State $\theta \leftarrow \theta + \alpha\Delta\cdot\ \nabla_{\theta}\Qhat_{\beta}(s,a;\theta)$
        \EndFor
        \State $\theta^{-} \leftarrow (1-\alpha^{-})\theta^{-} + \alpha^{-}\theta$
    \EndFor
\EndFor
\end{algorithmic}
\caption{Pseudo-code for RBF-DQN}\label{alg:RBFDQN}
\end{algorithm}

We now evaluate RBF-DQN against state-of-the-art value-function-only deep RL baselines. We understand NAF~\cite{naf} and ICNN~\cite{icnn} as two of the best continuous-action extensions of DQN in the RL literature, so we used them as our baselines. For RBF-DQN, as well as each of the two baselines, we performed 1000 updates per episodes when each episode ends. The authors of ICNN have released an official code base \footnote{github.com/locuslab/icnn}, which we used to obtain the ICNN learning curves. To the best of our knowledge, \citet{naf} did not release code, but we have used a public implementation of NAF \footnote{github.com/ikostrikov/pytorch-ddpg-naf}. Compared to the reported results in the ICNN paper, we were able to roughly achieve the same performance for NAF and ICNN. However, for completeness, we show, via dashed horizontal lines, levels of performance for NAF and ICNN reported by the ICNN paper whenever one was present. Another relevant baseline is CAQL~\cite{caql}, for which we could neither find the authors' code nor a public implementation. Since the original CAQL paper used environments with constrained action spaces, we only compare with the best reported results for Hopper-v3, HalfCheetah-v3, and Pendulum-v0, as these environments were used with the full action space. We have released our code: \footnote{github.com/kavosh8/RBFDQN\_pytorch}. Details of our hyper-parameter tuning can be found in the Appendix. RBF-DQN is clearly outperforming the two baselines, even when considering the results reported by the ICNN authors~\citep{icnn}.

\label{sec:RL_experiments}
\begin{figure*}[t]
  \centering
  \includegraphics[width=1\textwidth]{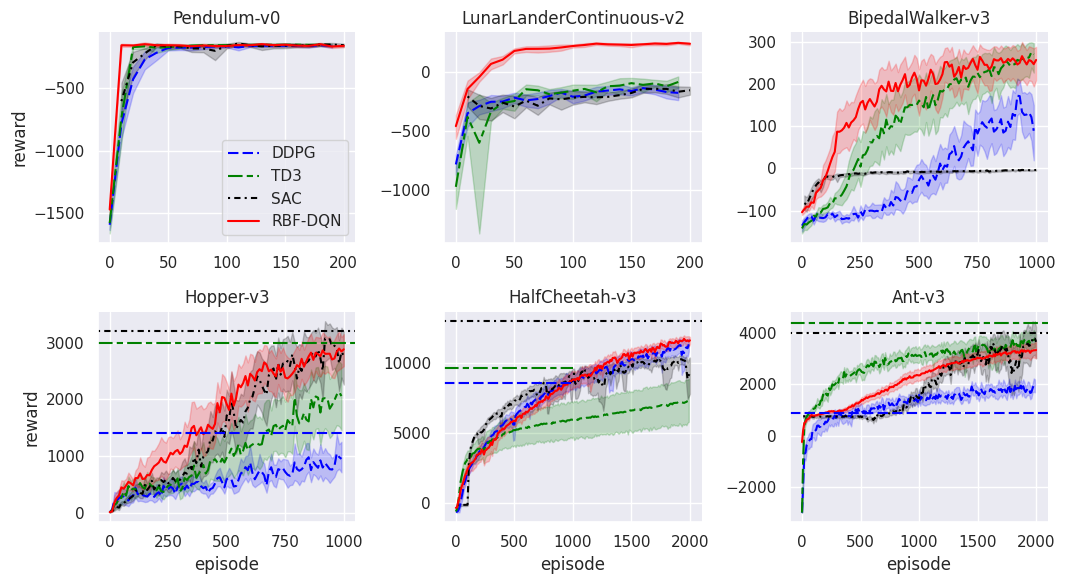}
    \caption{A comparison between RBF-DQN and state-of-the-art actor-critic deep RL.}
  \label{fig:RL_results_comparison_with_ac}
\end{figure*}

In light of the above results, we can claim that RBF-DQN is demonstrably a state-of-the-art value-function-only deep RL algorithm, but how does RBF-DQN compare to state-of-the-art actor-critic deep RL? To answer this question, we compared RBF-DQN to state-of-the-art actor critic deep RL, namely DDPG~\cite{dpg, ddpg}, TD3~\cite{TD3}, and SAC~\cite{soft_ac}. For TD3 and DDPG, we used the official code released by~\citet{TD3}\footnote{github.com/sfujim/TD3}, which is also used by numerous other papers. For SAC, we used a public implementation.\footnote{github.com/pranz24/pytorch-soft-actor-critic} In Figure~\ref{fig:RL_results_comparison_with_ac}, we show, not just the learning curves we obtained by running the implementations, but also results reported in the TD3 and SAC papers.

We note that the authors of TD3, DDPG, and SAC report their results in steps. In our experiments, we report results in episodes and have modified all implementations to have 1000 gradient updates per episode, while continuing to record steps. We plot horizontal lines at the level of reward obtained by the authors (and reported in their papers) at the average steps reached by the publicly available implementations when run for the specified number of episodes. For example, in Ant-v3, SAC reaches 1,076,890 steps on average when run for 2000 episodes. Therefore, we estimate the reward from the authors' graph at around $1.07 \times 10^{6}$ steps to be roughly 4000. More details for all relevant domains are available in the Appendix.

From Figure~\ref{fig:RL_results_comparison_with_ac}, RBF-DQN is performing better than or is competitive with state-of-the-art actor critic deep RL baselines. RBF-DQN can compete with these baselines despite the fact that it only uses 2 neural networks (an online value function and a target value function), while SAC and TD3, for instance, use 5 and 6 networks, respectively. In our experiments, TD3 and SAC perform 1000 gradient updates on two critics and one actor per each episode, while RBF-DQN only updates a deep RBVF 1000 times. Therefore, TD3 and SAC are also performing 3 times as many updates as RBF-DQN. Lastly, some of the ideas leveraged by TD3 and SAC, such as value clipping in TD3 and SAC, and entropy regularization in SAC, can be integrated into RBF-DQN. We believe these combinations are promising, and leave the investigation of these combinations for future work.

\subsection{Hyper-parameter Investigation}
\begin{figure}[h!]
  \centering
  \includegraphics[width=1\linewidth]{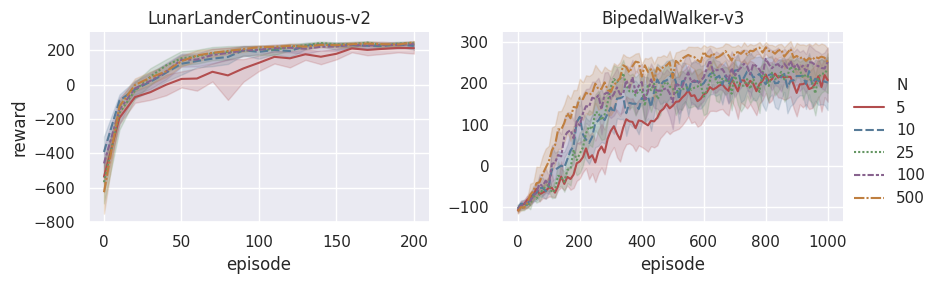}
    \caption{RBF-DQN performance as a function of the number of centroids $N$.}
  \label{fig:centroid_variation}
\end{figure}
\begin{figure}[h]
  \centering
  \includegraphics[width=1\linewidth]{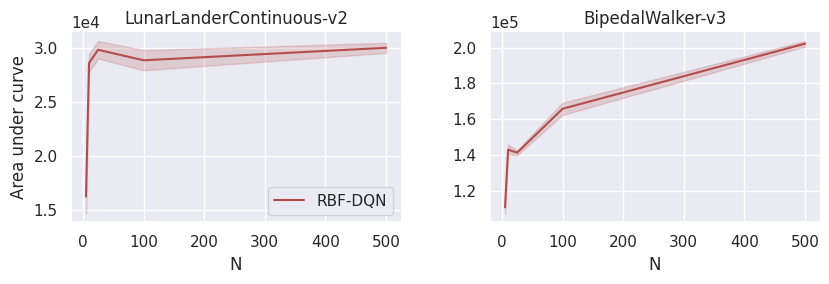}
    \caption{RBF-DQN performance (area under curve) as a function of the number of centroids $N$.}
  \label{fig:centroid_variation_auc}
\end{figure}

\begin{figure}[h]
  \centering
  \includegraphics[width=1\linewidth]{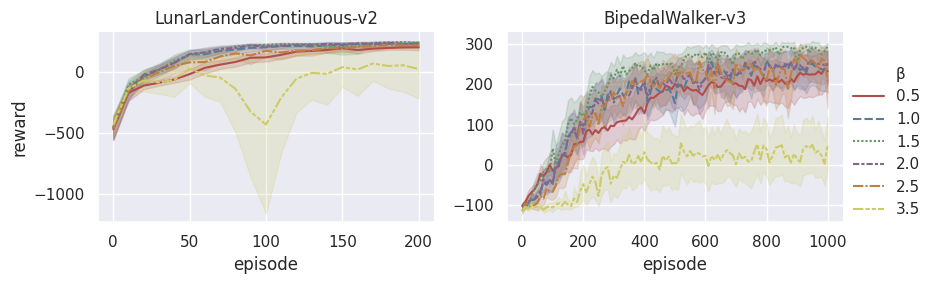}
    \caption{RBF-DQN performance as a function of the inverse smoothing parameter $\beta$.}
  \label{fig:beta_variation}
\end{figure}
\begin{figure}[h!]
  \centering
  \includegraphics[width=1\linewidth]{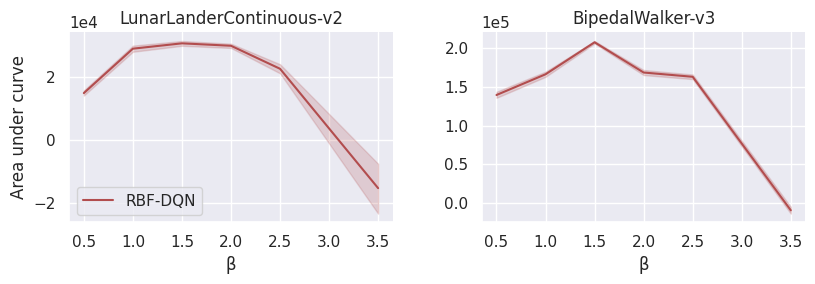}
    \caption{RBF-DQN performance (area under curve) as a function of the inverse smoothing parameter $\beta$.}
  \label{fig:beta_variation_auc}
\end{figure}

We introduced two new hyper-parameters for RBVFs, namely the number of centroids $N$, and the inverse smoothing parameter $\beta$. To better understand the impact of these hyper-parameters, we investigate the performance of the agent as a function of these hyper-parameters. We first show the performance of RBF-DQN as a function of $N$ (Figure~\ref{fig:centroid_variation} and Figure~\ref{fig:centroid_variation_auc}). We note two interesting observations. First, the agent can perform reasonably well even with a very small number of centroids such as $N=5$. Second, we see improved performance as we increase the number of centroids. The first observation suggests that we can get good results even when massive computation is not available and we have to use very small $N$. The second property is also appealing, because, in the presence of enough compute power, we can improve the agent's performance by simply increasing $N$.

In contrast, with respect to $\beta$, the inverse smoothing parameter of the negative exponential, we often see an inverted-U shape, with the best setting at some intermediate value (see Figure~\ref{fig:beta_variation} and Figure~\ref{fig:beta_variation_auc}). We note that, while using a large value of $\beta$ makes it theoretically possible to approximate any function up to any desired accuracy, the trade-off between using small $\beta$ values that make unwarranted generalization, and large $\beta$ values that yield extremely local approximations make intermediate values of $\beta$ work best, as demonstrated.
\subsection{RBVFs for Actor-Critic Deep RL}
So far, we have applied RBVFs to value-function-only RL, but can RBVFs be useful for other RL algorithms? To answer this question affirmatively, we now use RBVFs for actor-critic deep RL, in particular as the critic in the DDPG algorithm~\cite{dpg,ddpg}. Recall that, in contrast with value-function-only RL, actor-critic algorithms learn two separate networks: a value function (or the critic), and a policy (or the actor) that is mainly used for action selection. \citet{dpg} introduces the deterministic policy-gradient actor critic, but makes the observation that, in continuous control, computing the greedy action with respect to the learned critic is not tractable (see their Subsection 3.1), leading them to instead use a SARSA update for the critic. We showed that a deep RBVF approximately solves this maximization problem tractably. Leveraging this insight, we can modify the DDPG algorithm to use an RBVF critic. Given a tuple $\langle s,a,r,s' \rangle$, we can learn, via the critic, the value function with a Q-learning or SARSA update:
\begin{equation}
    \begin{split}
        &\theta\leftarrow\theta+\alpha\ \delta_{\textrm{RBF-DDPG}} \ \nabla_\theta \Qhat(s,a;\theta)\ ,\\
        &\quad \delta_{\textrm{Q-learning}} := r+\gamma \max_{i}\Qhat(s',a_i(s');\theta) -\Qhat(s,a;\theta)\\
        & \quad \delta_{\textrm{SARSA}} := r+\gamma
        \Qhat(s',\pi(s';\omega);\theta) -\Qhat(s,a;\theta)\ .
    \end{split}
    \label{eq:rbf_ddpg}
\end{equation}
\begin{figure}[h!]
  \centering
  \includegraphics[width=0.7\linewidth]{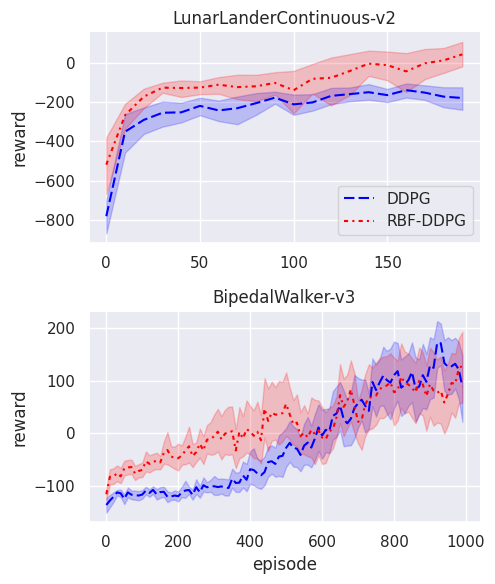}
    \caption{A comparison between DDPG and RBF-DDPG.}
  \label{fig:rbf_ddpg}
\end{figure}
The original DDPG algorithm used the SARSA update, obviating the action-maximization step. Using a deep RBVF, RBF-DDPG performs the Q-learning update shown above to update the critic. RBF-DDPG is otherwise analogous to DDPG.

We compare RBF-DDPG and DDPG in Figure~\ref{fig:rbf_ddpg}. It is clear that the use of RBVFs benefits DDPG. Although preliminary evaluations of RBF-DDPG do not exceed state-of-the-art performance, further investigation is required to see if the addition of other algorithmic ideas, such as those presented in \citet{TD3} and \citet{soft_ac}, can further improve the performance of RBF-DDPG. 
\section{Conclusion}
We introduced deep radial-basis value functions (RBVFs), and showed that they enable us to extend the Q-learning update from discrete-action settings to continuous control. Leaning on this insight, we introduced two new RL algorithms, RBF-DQN and RBF-DDPG. We showed that, in particular, RBF-DQN is significantly better than value-function-only deep RL baselines. RBF-DQN is also competitive with state-of-the-art actor-critic deep RL despite maintaining fewer networks and performing fewer updates. RBVFs facilitate action maximization, do not impede universal function approximation, and scale to large action spaces. Deep RBVFs are thus an appealing choice for value function approximation in continuous control.

\section{Future Work}
We envision several promising directions for future work. First, the RL literature has numerous examples of algorithmic ideas that help improve value-function-only algorithms (see \citet{rainbow} for some examples). These ideas are usually proposed for domains with discrete actions, so extending them to continuous-action domains using RBVFs could be an exciting direction to pursue. Our results set the stage to exploring these extensions.

We solely focused on deep RBVFs with negative exponentials, but various RBFs exist in the literature~\citep{reformulated_rbf}. We also envision future work where we learn an abstract action representation, akin to \citet{action_embedding_yash}, and utilize RBFs that operate in the abstract action space. This can better leverage the underlying structure of the action space.

Different exploration strategies exist for value-function-only RL: optimistic initialization \citep{RL_book,machado2015domain}, softmax policies \citep{rummery1994line,RL_book}, uncertainty-based exploration \citep{osband2016deep}, and Zooming \citep{zooming}. A combination of advanced exploration strategies with deep RBVFs could prove more effective than dithering strategies used in this paper.
\section{Acknowledgements}
We are appreciative of the feedback provided by many of our colleagues. In particular, we thank Dilip Arumugam for his comments in the early steps of the work, as well as Sam Lobel and Seungchan Kim for our invaluable discussions.
\bibliography{bibfile}
\onecolumn
\section{Appendix}
\subsection{Proofs}
\maxGap*
\begin{proof}
We begin by proving the first result. For an arbitrary action $a$, we can write $$\Qhat_{\beta}(s,a;\theta)=w_1 v_1(s;\theta) + ... + w_N v_N(s;\theta) \ ,$$
Without loss of generality, we sort centroids so that $\forall i\in[1,N-1],\ a_i~\leq~a_{i+1}$. Take two neighboring centroids $a_{L}$ and $a_{R}$ and notice that:
$$\forall i<L, \quad \frac{w_{L}}{w_i}=\frac{e^{-|a-a_L|}}{e^{-|a-a_{i}|}}=\frac{e^{-a+a_L}}{e^{-a+a_i}}=e^{a_L-a_i}\ \myeq \frac{1}{c_i}\ \implies w_i=w_L c_i\ .$$
In the above, we used the fact that all $a_i$ are to the left of $a$ and $a_L$. Similarly, we can argue that $\forall i>R\quad W_i=W_R c_i$. Intuitively, 
for actions between $a_L$ and $a_R$, we will have a constant ratio between the weight of a centroid to the left of $a_L$, over the weight of $a_L$ itself. The same holds for the centroids to the right of $a_R$.
In light of the above result, by renaming some variables we can now write:
\begin{eqnarray*}
		\Qhat_{\beta}(s,a;\theta)&=&w_1 v_1(s;\theta) + ... + w_L v_L(s;\theta) + w_R v_R(s;\theta) + ... + w_K v_K(s;\theta)\\
		&=&w_L c_1 v_1(s;\theta) + ... + w_L v_L(s;\theta) + w_R v_R(s;\theta) + ... + w_R c_K v_K(s;\theta) \\
		&=&w_L \big(c_1 v_1(s;\theta) + ... + v_L(s;\theta)\big) + w_R \big(v_R(s;\theta) + ... + c_K v_K(s;\theta)\big)\ .
\end{eqnarray*}
Moreover, note that the weights sum to 1: $w_L (c_1 + ... + 1) + w_R (1 + ... + c_K)=1\quad ,$
and $w_L$ is at its peak when we choose $a=a_L$ and at its smallest value when we choose $a=a_R$. A converse statement is true about $w_R$. Moreover, the weights monotonically increase and decrease as we move the input $a$. We call the endpoints of the range $w_{min}$ and $w_{max}$. As such, the problem $\max_{a\in[a_L,a_R]}\ \Qhat_{\beta}(s,a;\theta)$ could be written as this linear program:
\begin{eqnarray*}
	\max_{w_L,w_R} && w_L \big(c_1 v_1(s;\theta) + ... + v_L(s;\theta)\big) + w_R \big(v_R(s;\theta) + ... + c_K v_K(s;\theta)\big)\\
	s.t. && w_L (c_1 + ... + 1) + w_R (1 + ... + c_K)=1\\
	&& w_L,w_R\geq W_{min}\\
	&& w_L,w_R\leq W_{max}
\end{eqnarray*}
A standard result in linear programming is that every linear program has an extreme point that is an optimal solution \citep{convex}. Therefore, at least one of the points $(w_L=w_{min},w_R=w_{max})$ or $(w_L=w_{max},w_R=w_{min})$ is an optimal solution. It is easy to see that there is a one-to-one mapping between $a$ and $W_L,W_R$ in light of the monotonic property. As a result, the first point corresponds to the unique value of $a=a_R(s)$, and the second corresponds to unique value of $a=a_L(s)$. 
Since no point in between two centroids can be bigger than the surrounding centroids, at least one of the centroids is a globally optimal solution in the range $[a_1(s),a_N(s)]$, that is 

$$\quad \max_{a\in\fancyA} \Qhat_{\beta}(s,a;\theta)=\max_{i\in[1,N]} \Qhat_{\beta}(s,a_i;\theta)\ .$$

To finish the proof, we can show that $\forall a\!<\!a_1\ \Qhat_{\beta}(s,a;\theta)=\Qhat_{\beta}(s,a_1;\theta)$. The proof for $\forall a>a_N\ \Qhat_{\beta}(s,a;\theta)=\Qhat_{\beta}(s,a_N;\theta)$ follows similar steps:
\begin{eqnarray*}
	\forall a<a_1\ \Qhat_{\beta}(s,a;\theta)&=&\frac{\sum_{i=1}^N e^{-\beta \abs{a-a_{i}}}v_i(s)}{\sum_{i=1}^N e^{-\beta \abs{a-a_{i}}}}\\
	&=&\frac{\sum_{i=1}^N e^{-\beta \abs{a_1-c-a_i}}v_i(s)}{\sum_{i=1}^N e^{-\beta \abs{a_1-c-a_i}}}\quad (a=a_1-c \textrm{ for some } c>0)\\
	&=&\frac{\sum_{i=1}^N e^{\beta (a_1-c-a_i)}v_i(s)}{\sum_{i=1}^N e^{\beta (a_1-c-a_i)}}\quad (a_1-c<a_1\leq a_i)\\
	&=&\frac{e^{-c}\sum_{i=1}^N e^{\beta (a_1-a_i)}v_i(s)}{e^{-c}\sum_{i=1}^N e^{\beta (a_1-a_i)}}=\frac{\sum_{i=1}^N e^{\beta (a_1-a_i)}v_i(s)}{\sum_{i=1}^N e^{\beta (a_1-a_i)}}=\Qhat_{\beta}(s,a_1;\theta)\ ,
\end{eqnarray*}
We now prove the second result of the Theorem. First,
let $i^{*}$ be the index corresponding to the centroid with the highest centroid value:
$$i^{*} =\arg\max_{i} v_{i}(s;\theta)$$
Now observe that:
	\begin{eqnarray*}
		\max_{a\in\fancyA} \Qhat_{\beta}(s,a;\theta)-\max_{i\in [1:N]} \Qhat_{\beta}(s,a_i;\theta)&\leq& v_{i^{*}}(s;\theta)-\max_{i\in [1:N]} \Qhat_{\beta}(s,a_i;\theta)\\
		&\leq& v_{i^{*}}(s;\theta)-\Qhat_{\beta}(s,a_{i^*};\theta)\ .
	\end{eqnarray*}
The first inequality holds because $\max_{a\in\fancyA} \Qhat_{\beta}(s,a;\theta)\leq v_{i^{*}}(s;\theta)$. This is true because $\Qhat_{\beta}$ is just a weighted average of all $v_{i}$ values, and so, it cannot be bigger than the largest among them. Also, $\max_{i\in [1:N]} \Qhat_{\beta}(s,a_i;\theta)\geq \Qhat_{\beta}(s,a_{i^*};\theta)$, hence the second inequality. \\
For the rest of the proof we assume, without loss of generality, that $i^{*}=1$. We also drop the dependence on $\theta$:
	\begin{eqnarray*}
	    \max_{a\in\fancyA} \Qhat_{\beta}(s,a)-\max_{i\in [1:N]} \Qhat_{\beta}(s,a_i)&=&
		v_{1}(s)-\Qhat_{\beta}(s,a_{1})\\
		&=&v_{1}(s) - \frac{\sum_{i=1}^N e^{-\beta \norm{a_{1}-a_i}}v_i(s)}{\sum_{i=1}^N e^{-\beta \norm{a_{1}-a_i}}}\\
		&=&\frac{\sum_{i=1}^N e^{-\beta \norm{a_{1}-a_i}}\big(v_1(s)-v_i(s)\big)}{\sum_{i=1}^N e^{-\beta \norm{a_{1}-a_i}}}\\
		&=&\frac{\sum_{i=2}^N e^{-\beta \norm{a_{1}-a_i}}\big(v_1(s)-v_i(s)\big)}{1+\sum_{k=2}^K e^{-\beta \norm{a_{1}-a_i}}}\\
		&\leq &\Delta\frac{\sum_{i=2}^N e^{-\beta \norm{a_{1}-a_i}}}{1+\sum_{i=2}^N e^{-\beta \norm{a_{1}-a_i}}}\quad \textrm{where} \ \Delta=v_{1}(s)-\min_{i} v_{i}(s)\\
		&\leq &\Delta\sum_{i=2}^N\frac{ e^{-\beta \norm{a_{1}-a_i}}}{1+ e^{-\beta \norm{a_{1}-a_i}}}\quad \textrm{(See ~\citet{song2019revisiting})}\\
		&\leq&\Delta\sum_{i=2}^N\frac{ 1}{1+ e^{\beta \norm{a_{1}-a_i}}}=\fancyO(e^{-\beta}).
	\end{eqnarray*}\end{proof}
	
\begin{restatable}{theorem}{UFATheorem}
\label{theorem:ufa}
 Consider any state--action value function $Q^{\pi}(s,a)$ defined on a closed action space $\fancyA$. Assume that $Q^{\pi}(s,a)$ is a continuous function. For a fixed state $s$ and for any $\epsilon>0$, there exists a deep RBF value function $\Qhat_{\beta}(s,a;\theta)$ and a setting of the smoothing parameter $\beta_0$ for which:
$$\forall a\in \fancyA\quad \forall \beta\geq \beta_0 \quad |Q^{\pi}(s,a)-\Qhat_{\beta}(s,a;\theta)|\leq\epsilon \ .$$
\end{restatable}
\begin{proof}
\citet{normalized_RBF_UFA} provides a proof, but for completeness, we also provide a simplified proof. Since $Q^\pi$ is continuous, we leverage the fact that it is Lipschitz with a Lipschitz constant $L$:
$$\forall a_0,\ a_1\quad |f(a_1)-f(a_0)|\leq L\norm{a_1-a_0}$$
As such, assuming that $\norm{a_1-a_0}\leq \frac{\epsilon}{4L}$, we have that $|f(a_1)-f(a_0)|\leq \frac{\epsilon}{4}$
Consider a set of centroids $\{c_1,c_2,...,c_N\}$, define the $cell(j)$ as:
$$ cell(j)=\{a\in\fancyA|\ \norm{a-c_j}= \min_{z} \norm{a-c_z}\}\ ,$$
and the radius $Rad(j,\fancyA)$ as:
$$Rad(j,\fancyA):=\sup_{x\in cell(j)}\norm{x-c_j}\ .$$
Assuming that $\fancyA$ is a closed set, there always exists a set of centroids $\{c_1,c_2,...,c_N\}$ for which $Rad(c,\fancyA)\leq \frac{\epsilon}{4L}$. Now consider the following functional form:
\begin{eqnarray*}
\Qhat_{\beta}(s,a)&:=&\sum_{j=1}^N Q^\pi(s,c_j)w_j\ ,\\
&&\textrm{where}\quad w_j=\frac{e^{-\beta\norm{a-c_j}}}{\sum_{z=1}^N e^{-\beta\norm{a-c_z}}}\ .
\end{eqnarray*}
Now suppose $a$ lies in a subset of cells, called the \textit{central} cells $\mathcal{C}$:
$$\mathcal{C}:=\{j|a\in cell(j)\}\ ,$$
We define a second \textit{neighboring} set of cells:
$$\mathcal{N}:=\{j|cell(j)\cap \big(\cup_{i\in \mathcal{C}} cell(i)\big)\neq\emptyset \} -\mathcal{C}\ ,$$
and a third set of \textit{far} cells:
$$\mathcal{F}:=\{j|j\notin \mathcal{C}\ \&\ j\notin \mathcal{N} \}\ ,$$
We now have:
\begin{eqnarray*}
|Q^{\pi}(s,a)-\Qhat_{\beta}(s,a;\theta)|&=&|\sum_{j=1}^{N}\big(Q^{\pi}(s,a)-Q^{\pi}(s,c_j)\big)w_j|\leq\sum_{j=1}^{N}\big|Q^{\pi}(s,a)-Q^{\pi}(s,c_j)\big|w_j\\
&=&\sum_{j\in \mathcal C}\big|Q^{\pi}(s,a)-Q^{\pi}(s,c_j)\big|w_j + \sum_{j\in \mathcal N}\big|Q^{\pi}(s,a)-Q^{\pi}(s,c_j)\big|w_j\\
&&+\sum_{j\in \mathcal F}\big|Q^{\pi}(s,a)-Q^{\pi}(s,c_j)\big|w_j
\end{eqnarray*}
We now bound each of the three sums above. Starting from the first sum, it is easy to see that $\big|Q^{\pi}(s,a)-Q^{\pi}(s,c_j)\big|\leq \frac{\epsilon}{4} $, simply because $a\in cell(j)$. As for the second sum, since $c_j$ is the centroid of a neighboring cell, using a central cell $i$, we can write:
$$\norm{a-c_j}= \norm{a-c_i+c_i-c_j}\leq\norm{a-c_i}+\norm{c_i-c_j}\leq \frac{\epsilon}{4L}+\frac{\epsilon}{4L}=\frac{\epsilon}{2L}\ ,$$
and so in this case $\big|Q^{\pi}(s,a)-\Qhat_{\beta}(s,c_j)\big|\leq \frac{\epsilon}{2}$.
In the third case with the set of far cells $\mathcal{F}$, observe that for a far cell $j$ and a central cell $i$ we have:
$$\frac{w_j}{w_i}=\frac{e^{-\beta\norm{a-c_j}}}{e^{-\beta\norm{a-c_i}} } \rightarrow w_j=w_i e^{-\beta(\norm{a-c_j}-\norm{a-c_i})}\leq w_i e^{-\beta\mu}\leq e^{-\beta\mu}, $$
For some $\mu>0$. In the above, we used the fact that $\norm{a-c_j}-\norm{a-c_i}> 0$ is always true. Then:
\begin{eqnarray*}
&&|Q^{\pi}(s,a)-\Qhat_{\beta}(s,a)|\\
&=&\sum_{j\in \mathcal C}\underbrace{\big|Q^{\pi}(s,a)-Q^{\pi}(s,c_j)\big|}_{\leq\frac{\epsilon}{4}}\underbrace{w_j}_{\leq 1} + \sum_{j\in \mathcal N}\underbrace{\big|Q^{\pi}(s,a)-Q^{\pi}(s,c_j)\big|}_{\leq\frac{\epsilon}{2}}\underbrace{w_j}_{1}\\\
&&+\sum_{j\in \mathcal F}\big|Q^{\pi}(s,a)-Q^{\pi}(s,c_j)\big|\underbrace{w_j}_{e^{-\beta\mu}}\\
&\leq&\frac{\epsilon}{4}+\frac{\epsilon}{2}+\sum_{j\in \mathcal F}\big|Q^{\pi}(s,a)-Q^{\pi}(s,c_j)\big|e^{-\beta\mu}\\
&\leq&\frac{\epsilon}{4}+\frac{\epsilon}{2}+2N\sup_{a}|Q^{\pi}(s,a)| e^{-\beta\mu}
\end{eqnarray*}
In order to have $2N\sup_{a}|Q^{\pi}(s,a)| e^{-\beta\mu}\leq \frac{\epsilon}{4}$, it suffices to have $\beta\geq\frac{-1}{\mu}\log(\frac{\epsilon}{8N\sup_{a}|Q^{\pi}(s,a)|}):=\beta_0$. To conclude the proof:
$$|Q^{\pi}(s,a)-\Qhat_{\beta}(s,a;\theta)|\leq\epsilon \quad\forall\ \beta\geq\beta_0\ .$$\end{proof}

\subsection{Hyperparameter Tuning Deep RBVFs}
The following table lists the common hyper-parameters used for deep RBVFs in RBF-DQN and RBF-DDPG.
\begin{figure}[h]
\begin{center}
\begin{tabular}{ c c} 
hyper-parameter & value \\
 \hline
number of hidden layers for centroid values & 3 \\ 
number of hidden layers for centroid locations & 1 \\ 
number of nodes in all hidden layers & 512 \\
target network learning rate (exponential moving average) & 0.005\\
replay buffer size & $5\times 10^5$\\
discount-rate $\gamma$ & 0.99\\
size of mini-batch & 256 \\
number of centroids $N$ & 100 \\
optimizer & RMSProp \\
number of updates per episode & $10^3$ \\
\hline
\end{tabular}

\label{table:RBVF_common_hyper_parameters}
\caption{Common hyper-parameters used for deep RBVFs in RBF-DQN and RBF-DDPG.}
\end{center}
\end{figure}
For RBF-DQN, we tuned the learning rates and the inverse smoothing parameter via grid search. The best settings for each domain is listed below: 
\begin{figure}[h!]
\begin{center}
\begin{tabular}{ c c} 
domain inverse smoothing \\
 \hline
1 \\ 
2 \\ 
2  \\ 
1.5  \\ 
.25 \\ 
.1 \\ 
\hline
\end{tabular}
\label{table:RBF_DQN_hyper_parameters}
\caption{Tuned RBF-DQN inverse smoothing for each domain.}
\end{center}
\end{figure}

\subsection{Reported results from papers for baselines}

\subsubsection{ICNN, NAF, and CAQL}
For the domains we include, the authors report results on some domains. We rely on the ICNN authors' results for NAF as the original paper does not use OpenAI Gym environments. ICNN and NAF report results for Hopper-v3, HalfCheetah-v3, and Ant-v3. CAQL reports results on Pendulum-v0, Hopper-v3, and HalfCheetah-v3. For these baselines, we report the highest score achieved by any variation of the algorithm from the paper. For ICNN and NAF, no mention of steps or episodes is included, so we simply take the highest score. For CAQL, Pendulum-v0 is ran for 200k steps, which is significantly higher than 200 episodes (equivalent to 40000 steps as each episode is exactly 200 steps). Meanwhile, Hopper-v3 and HalfCheetah-v3 are ran for 500k steps. Since there is no available implementation, we simply report the highest score across their many variants. 

\subsubsection{DDPG, TD3, and SAC}
We manage to include the step count when run for the specified episodes for the SAC implementation. For the three domains, we include the average step count here:

\begin{figure}[h!]
\begin{center}
\begin{tabular}{ c c c c} 
domain & episodes & steps \\
\hline
Hopper-v3 & 1000 & 518494 \\ 
HalfCheetah-v3 & 2000 & 2000000 \\ 
Ant-v3 & 2000 & 1076890 \\ 
\hline
\end{tabular}
\label{table:SAC_steps}
\caption{Average steps reached by SAC}
\end{center}
\end{figure}

Additionally, aside from HalfCheetah-v3 (which always has a fixed episode length of 1000 steps), we manage to reach the claimed performance for Hopper-v3 and Ant-v3, enabling us to trust our step counts to be representative. 

Since the authors' implementations did not include step counts for TD3 and DDPG, we use SAC to estimate the step count for Hopper-v3 and Ant-v3 (as HalfCheetah-v3 is exactly known) for a fair comparison. We also note that it was difficult to reproduce the TD3 results even with the authors' code. We highlight a comment in the README, "Code is no longer exactly representative of the code used in the paper.", as a possible explanation.
\end{document}


\onecolumn
\section{Proofs}
\begin{theorem}
Let $\Qhat_{\beta}$ be a member of the class of normalized Gaussian RBF value functions.
\begin{enumerate}[I)]
    \item For a one-dimensional action space $\mathcal{A}=\mathcal{R}$: $$\max_{a\in\fancyA} \Qhat_{\beta}(s,a;\theta)=\max_{i\in[1,N]} \Qhat_{\beta}(s,a_i;\theta)\ .$$
    \item For $\fancyA=\fancyR^d\quad \forall d\geq 1$:
    $$0\leq\max_{a\in\fancyA} \Qhat_{\beta}(s,a;\theta)-\!\max_{i\in[1,N]} \Qhat_{\beta}(s,a_i;\theta)\leq \mathcal{O}(e^{-\beta}) \ .$$
\end{enumerate}
\label{theorem:maxGap}
\end{theorem}
\begin{proof}
We begin by proving the first result. For an arbitrary action $a$, we can write:
$$\Qhat_{\beta}(s,a;\theta)=w_1 v_1(s;\theta) + ... + w_N v_N(s;\theta) \ ,$$
where each weight $w_i$ is determined via softmax. Without loss of generality, we sort all anchor points so that $\forall i\ a_i<a_{i+1}$. Take two neighboring centroids $a_{L}$ and $a_{R}$ and notice that:
$$\forall i<L, \quad \frac{w_{L}}{w_i}=\frac{e^{-|a-a_L|}}{e^{-|a-a_{i}|}}=\frac{e^{-a+a_L}}{e^{-a+a_i}}=e^{a_L-a_i}\ \myeq \frac{1}{c_i}\ \implies w_i=w_L c_i\ .$$
In the above, we used the fact that all $a_i$ are to the left of $a$ and $a_L$. Similarly, we can argue that $\forall i>R\quad W_i=W_R c_i$. Intuitively, 
as long as the action is between $a_L$ and $a_R$, the ratio of the weight of a centroid to the left of $a_L$, over the weight of $a_L$ itself, remains constant and does not change with $a$. The same holds for the centroids to the right of $a_R$.
In light of the above result, by renaming some variables we can now write:
\begin{eqnarray*}
		\Qhat_{\beta}(s,a;\theta)&=&w_1 v_1(s;\theta) + ... + w_L v_L(s;\theta) + w_R v_R(s;\theta) + ... + w_K v_K(s;\theta)\\
		&=&w_L c_1 v_1(s;\theta) + ... + w_L v_L(s;\theta) + w_R v_R(s;\theta) + ... + w_R c_K v_K(s;\theta) \\
		&=&w_L \big(c_1 v_1(s;\theta) + ... + v_L(s;\theta)\big) + w_R (v_R(s;\theta) + ... + c_K v_K(s;\theta))\ .
\end{eqnarray*}
Moreover, note that the weights need to sum up to 1:
$$w_L (c_1 + ... + 1) + w_R (1 + ... + c_K)=1\quad ,$$
and $w_L$ is at its peak when we choose $a=a_L$ and at its smallest value when we choose $a=a_R$. A converse statement is true about $w_R$. Moreover, the weights monotonically increase and decrease as we move the input $a$. We call the endpoints of the range $w_{min}$ and $w_{max}$. As such, the problem $\max_{a\in[a_L,a_R]}\ \Qhat_{\beta}(s,a;\theta)$ could be written as this linear program:
\begin{eqnarray*}
	\max_{w_L,w_R} && w_L \big(c_1 v_1(s;\theta) + ... + v_L(s;\theta)\big) + w_R \big(v_R(s;\theta) + ... + c_K v_K(s;\theta)\big)\\
	s.t. && w_L (c_1 + ... + 1) + w_R (1 + ... + c_K)=1\\
	&& w_L,w_R\geq W_{min}\\
	&& w_L,w_R\leq W_{max}
\end{eqnarray*}
A standard result in linear programming is that every linear program has an extreme point that is an optimal solution \citep{convex}. Therefore, at least one of the points $(w_L=w_{min},w_R=w_{max})$ or $(w_L=w_{max},w_R=w_{min})$ is an optimal solution. It is easy to see that there is a one-to-one mapping between $a$ and $W_L,W_R$ in light of the monotonic property. As a result, the first point corresponds to the unique value of $a=a_R(s)$, and the second corresponds to unique value of $a=a_L(s)$. 
Since no point in between two centroids can be bigger than the surrounding centroids, at least one of the centroids is a globally optimal solution in the range $[a_1(s),a_N(s)]$, that is $$\max_{a\in[a_1(s;\theta),a_N(s;\theta)]}\Qhat_{\beta}(s,a;\theta)=\max_{a_i}\Qhat_{\beta}(s,a_i;\theta). $$ \\

To finish the proof, we can show that $\forall a<a_1\ \Qhat_{\beta}(s,a;\theta)=\Qhat_{\beta}(s,a_1;\theta)$. The proof for $\forall a>a_N\ \Qhat_{\beta}(s,a;\theta)=\Qhat_{\beta}(s,a_N;\theta)$ follows similar steps. So,
\begin{eqnarray*}
	\forall a<a_1\ \Qhat_{\beta}(s,a;\theta)&=&\frac{\sum_{i=1}^N e^{-\beta \abs{a-a_{i}(s)}}v_i(s)}{\sum_{i=1}^N e^{-\beta \abs{a-a_{i}(s)}}}\\
	&=&\frac{\sum_{i=1}^N e^{-\beta \abs{a_1-c-a_i(s)}}v_i(s)}{\sum_{i=1}^N e^{-\beta \abs{a_1-c-a_i(s)}}}\\
	&=&\frac{\sum_{i=1}^N e^{\beta (a_1-c-a_i(s))}v_i(s)}{\sum_{i=1}^N e^{\beta (a_1-c-a_i(s))}}\\
	&=&\frac{e^{-c}\sum_{i=1}^N e^{\beta (a_1-a_i(s))}v_i(s)}{e^{-c}\sum_{i=1}^N e^{\beta (a_1-a_i(s))}}\\
	&=&\frac{\sum_{i=1}^N e^{\beta (a_1-a_i(s))}v_i(s)}{\sum_{i=1}^N e^{\beta (a_1-a_i(s))}}=\Qhat_{\beta}(s,a_1;\theta)\ ,
\end{eqnarray*}
which concludes the proof of the first part.

We now move to the more general case with $\fancyA=\fancyR^m$:
	\begin{eqnarray*}
		\max_{a} \Qhat_{\beta}(s,a;\theta)-\max_{i\in \{1:N\}} \Qhat(s,a_i;\theta)&\leq& v_{max}(s;\theta)-\max_{i\in \{1:N\}} \Qhat(s,a_i;\theta)\\
		&\leq& v_{max}(s;\theta)-\Qhat_{\beta}(s,a_{max};\theta)\ .
	\end{eqnarray*}
WLOG, we assume the first centroid is the one with highest $v$, that is $v_1(s;\theta)=\arg \max_{v_i} v_i(s;\theta)$, and conclude the proof. Note that a related result was shown recently~\cite{song2019revisiting}:
	\begin{eqnarray*}
		v_{max}(s)-\Qhat_{\beta}(s,a_{max};\theta)&=&v_{1} - \frac{\sum_{i=1}^N e^{-\beta \norm{a_{1}-a_i(s)}}v_i(s)}{\sum_{i=1}^N e^{-\beta \norm{a_{1}-a_i(s)}}}\\
		&=&\frac{\sum_{i=1}^N e^{-\beta \norm{a_{1}-a_i(s)}}\big(v_1(s)-v_i(s)\big)}{\sum_{i=1}^N e^{-\beta \norm{a_{1}-a_i(s)}}}\\
		&=&\frac{\sum_{i=2}^N e^{-\beta \norm{a_{1}-a_i(s)}}\big(v_1(s)-v_i(s)\big)}{1+\sum_{k=2}^K e^{-\beta \norm{a_{1}-a_i(s)}}}\\
		&\leq &\Delta_q\frac{\sum_{i=2}^N e^{-\beta \norm{a_{1}-a_i(s)}}}{1+\sum_{i=2}^N e^{-\beta \norm{a_{1}-a_i(s)}}}\\
		&\leq &\Delta_q\sum_{i=2}^N\frac{ e^{-\beta \norm{a_{1}-a_i(s)}}}{1+ e^{-\beta \norm{a_{1}-a_i(s)}}}\\
		&=&\Delta_q\sum_{i=2}^N\frac{ 1}{1+ e^{\beta \norm{a_{1}-a_i(s)}}}=\fancyO(e^{-\beta}).
	\end{eqnarray*}\end{proof}
	
\begin{theorem}
Consider any state--action value function $Q^{\pi}(s,a)$ defined on a closed action space $\fancyA$. Assume that $Q^{\pi}(s,a)$ is a continuous function. For a fixed state $s$ and for any $\epsilon>0$, there exists a deep RBF value function $\Qhat_{\beta}(s,a;\theta)$ and a setting of the smoothing parameter $\beta_0$ for which:
$$\forall a\in \fancyA\quad \forall \beta\geq \beta_0 \quad |Q^{\pi}(s,a)-\Qhat_{\beta}(s,a;\theta)|\leq\epsilon \ .$$
\label{theorem:ufa}
\end{theorem}
\begin{proof}
Since $Q^\pi$ is continuous, we leverage the fact that it is Lipschitz with a Lipschitz constant $L$:
$$\forall a_0,\ a_1\quad |f(a_1)-f(a_0)|\leq L\norm{a_1-a_0}$$
As such, assuming that $\norm{a_1-a_0}\leq \frac{\epsilon}{4L}$, we have that 
\begin{equation}
    |f(a_1)-f(a_0)|\leq \frac{\epsilon}{4}
\end{equation}
Consider a set of centroids $\{c_1,c_2,...,c_N\}$, define the $cell(j)$ as:
$$ cell(j)=\{a\in\fancyA|\ \norm{a-c_j}= \min_{z} \norm{a-c_z}\}\ ,$$
and the radius $Rad(j,\fancyA)$ as:
$$Rad(j,\fancyA):=\sup_{x\in cell(j)}\norm{x-c_j}\ .$$
Assuming that $\fancyA$ is a closed set, there always exists a set of centroids $\{c_1,c_2,...,c_N\}$ for which $Rad(c,\fancyA)\leq \frac{\epsilon}{4L}$. Now consider the following functional form:
\begin{eqnarray*}
\Qhat_{\beta}(s,a)&:=&\sum_{j=1}^N Q^\pi(s,c_j)w_j\ ,\\
&&\textrm{where}\quad w_j=\frac{e^{-\beta\norm{a-c_j}}}{\sum_{z=1}^N e^{-\beta\norm{a-c_z}}}\ .
\end{eqnarray*}
Now suppose $a$ lies in a subset of cells, called the \textit{central} cells $\mathcal{C}$:
$$\mathcal{C}:=\{j|a\in cell(j)\}\ ,$$
We define a second \textit{neighboring} set of cells:
$$\mathcal{N}:=\{j|cell(j)\cap \big(\cup_{i\in \mathcal{C}} cell(i)\big)\neq\emptyset \} -\mathcal{C}\ ,$$
and a third set of \textit{far} cells:
$$\mathcal{F}:=\{j|j\notin \mathcal{C}\ \&\ j\notin \mathcal{N} \}\ ,$$
We now have:
\begin{eqnarray*}
|Q^{\pi}(s,a)-\Qhat_{\beta}(s,a;\theta)|&=&|\sum_{j=1}^{N}\big(Q^{\pi}(s,a)-Q^{\pi}(s,c_j)\big)w_j|\\
&\leq&\sum_{j=1}^{N}\big|Q^{\pi}(s,a)-Q^{\pi}(s,c_j)\big|w_j\\
&=&\sum_{j\in \mathcal C}\big|Q^{\pi}(s,a)-Q^{\pi}(s,c_j)\big|w_j + \sum_{j\in \mathcal N}\big|Q^{\pi}(s,a)-Q^{\pi}(s,c_j)\big|w_j+\sum_{j\in \mathcal F}\big|Q^{\pi}(s,a)-Q^{\pi}(s,c_j)\big|w_j
\end{eqnarray*}
We now bound each of the three sums above. Starting from the first sum, it is easy to see that $\big|Q^{\pi}(s,a)-Q^{\pi}(s,c_j)\big|\leq \frac{\epsilon}{4} $, simply because $a\in cell(j)$. As for the second sum, since $c_j$ is the centroid of a neighboring cell, using a central cell $i$, we can write:
$$\norm{a-c_j}= \norm{a-c_i+c_i-c_j}\leq\norm{a-c_i}+\norm{c_i-c_j}\leq \frac{\epsilon}{4L}+\frac{\epsilon}{4L}=\frac{\epsilon}{2L}\ ,$$
and so in this case $\big|Q^{\pi}(s,a)-\Qhat_{\beta}(s,c_j)\big|\leq \frac{\epsilon}{2}$.
In the third case with the set of far cells $\mathcal{F}$, observe that for a far cell $j$ and a central cell $i$ we have:
$$\frac{w_j}{w_i}=\frac{e^{-\beta\norm{a-c_j}}}{e^{-\beta\norm{a-c_i}} } \rightarrow w_j=w_i e^{-\beta(\norm{a-c_j}-\norm{a-c_i})}\leq w_i e^{-\beta\mu}\leq e^{-\beta\mu}, $$
For some $\mu>0$. In the above, we used the fact that $\norm{a-c_j}-\norm{a-c_i}> 0$ is always true.

Putting it all together, we have:
\begin{eqnarray*}
&&|Q^{\pi}(s,a)-\Qhat_{\beta}(s,a)|\\
&=&\sum_{j\in \mathcal C}\underbrace{\big|Q^{\pi}(s,a)-Q^{\pi}(s,c_j)\big|}_{\leq\frac{\epsilon}{4}}\underbrace{w_j}_{\leq 1} + \sum_{j\in \mathcal N}\underbrace{\big|Q^{\pi}(s,a)-Q^{\pi}(s,c_j)\big|}_{\leq\frac{\epsilon}{2}}\underbrace{w_j}_{1}+\sum_{j\in \mathcal F}\big|Q^{\pi}(s,a)-Q^{\pi}(s,c_j)\big|\underbrace{w_j}_{e^{-\beta\mu}}\\
&\leq&\frac{\epsilon}{4}+\frac{\epsilon}{2}+\sum_{j\in \mathcal F}\big|Q^{\pi}(s,a)-Q^{\pi}(s,c_j)\big|e^{-\beta\mu}\\
&\leq&\frac{\epsilon}{4}+\frac{\epsilon}{2}+2N\sup_{a}|Q^{\pi}(s,a)| e^{-\beta\mu}
\end{eqnarray*}
In order to have $2N\sup_{a}|Q^{\pi}(s,a)| e^{-\beta\mu}\leq \frac{\epsilon}{4}$, it suffices to have $\beta\geq\frac{-1}{\mu}\log(\frac{\epsilon}{8N\sup_{a}|Q^{\pi}(s,a)|}):=\beta_0$. To conclude the proof:
$$|Q^{\pi}(s,a)-\Qhat_{\beta}(s,a;\theta)|\leq\epsilon \quad\forall\ \beta\geq\beta_0\ .$$
For a similar proof, see \cite{normalized_RBF_UFA}. \end{proof}
\section{Hyper-parameter Tuning}
\subsection{Common Hyper-parameters}
For all value-function-based baselines, per one episode, we performed 100 updates to their value network. For DDPG, we performed 100 updates to the value network, and 100 updates to the policy network. For TD3, we performed 50 updates to the policy due to its delayed policy updates. The target network parameters were updated in all baselines using a step size $\alpha_{\theta^-}=0.005$. The maximum length of the replay buffer was fixed to 500000. Rewards were always clipped to the range $[-20,20]$. We used a batch size 256. In terms of network topology, we tuned the number of nodes per layer and the number of hidden layers of each network, and chose the one that performed robustly across all 9 domains. Specifically we tried 1,2, and 3 hidden layers each having 128 or 512 nodes. For DDPG, TD3, and input-convex neural nets, we also tried network topology parameters according to the settings suggested in the original papers.

To determine the best hyper-parameters, we ran each algorithm for 3 independent runs using the chosen hyper-parameter setting, and selected the one that performed best on average. Best performance was defined as the highest average episodic return after training was done. We now move to domain-dependent hyper-parameters.
\subsection{RBF-DQN}
For each domain, we tuned the smoothing parameter $\beta$ using random search \cite{deep_learning_book} from the range $[0.1,3]$. We also tuned the learning rate for RMSProp using random search from the range $[5\times 10^{-6},5\times 10^{-2}]$. All other hyper-parameters were kept constant across different domains.
\subsection{Feed-forward Network}
For each domain, we tuned tuned the learning rate for RMSProp using random search from the range $[5\times 10^{-6},5\times 10^{-2}]$. We also tuned the parameters of gradient ascent optimizer for solving $\max_{a\in\fancyA} \Qhat(s,a;\theta)$. In particular, there were two parameters, namely a step size that was tuned using random search from the range $[0.0001,0.1]$, and the number of gradient ascent steps which was tuned using grid search from $\{10,20,50\}$. All other hyper-parameters were kept constant across different domains.
\subsection{Wire Fitting}
For each domain, we tuned the learning rate for RMSProp using random search from the range $[5\times 10^{-6},5\times 10^{-2}]$. All other hyper-parameters were kept constant across different domains.
\subsection{Input-Convex Neural Network}
For each domain, we tuned tuned the learning rate for RMSProp using random search from the range $[5\times 10^{-6},5\times 10^{-2}]$. All other hyper-parameters were kept constant across different domains.
\subsection{DDPG}
For each domain, we tuned the two learning rates from the range $[5\times 10^{-6},5\times 10^{-2}]$. Note that the learning rates for the value network and the policy network were tuned separately. For each individual domain, we tried two optimizers, namely RMSProp and Adam \cite{adam}. All other hyper-parameters were kept constant across different domains.
\subsection{TD3}
Similar to DDPG, we tuned the two learning rates for the value network and the policy network using the same range. We again tried two optimizers, namely RMSProp and Adam \cite{adam}. \citet{TD3} also introduced a $\sigma$ parameter for target policy regularization in TD3. We tuned this hyper-parameter for each domain as well using grid search and from $\{0.01,0.05,0.1,0.2\}$. All other hyper-parameters were kept constant across different domains.
\newpage
\section{A Comparison based on Final Performance}
\begin{figure*}[h]
    \centering
    \includegraphics[width=0.95\columnwidth]{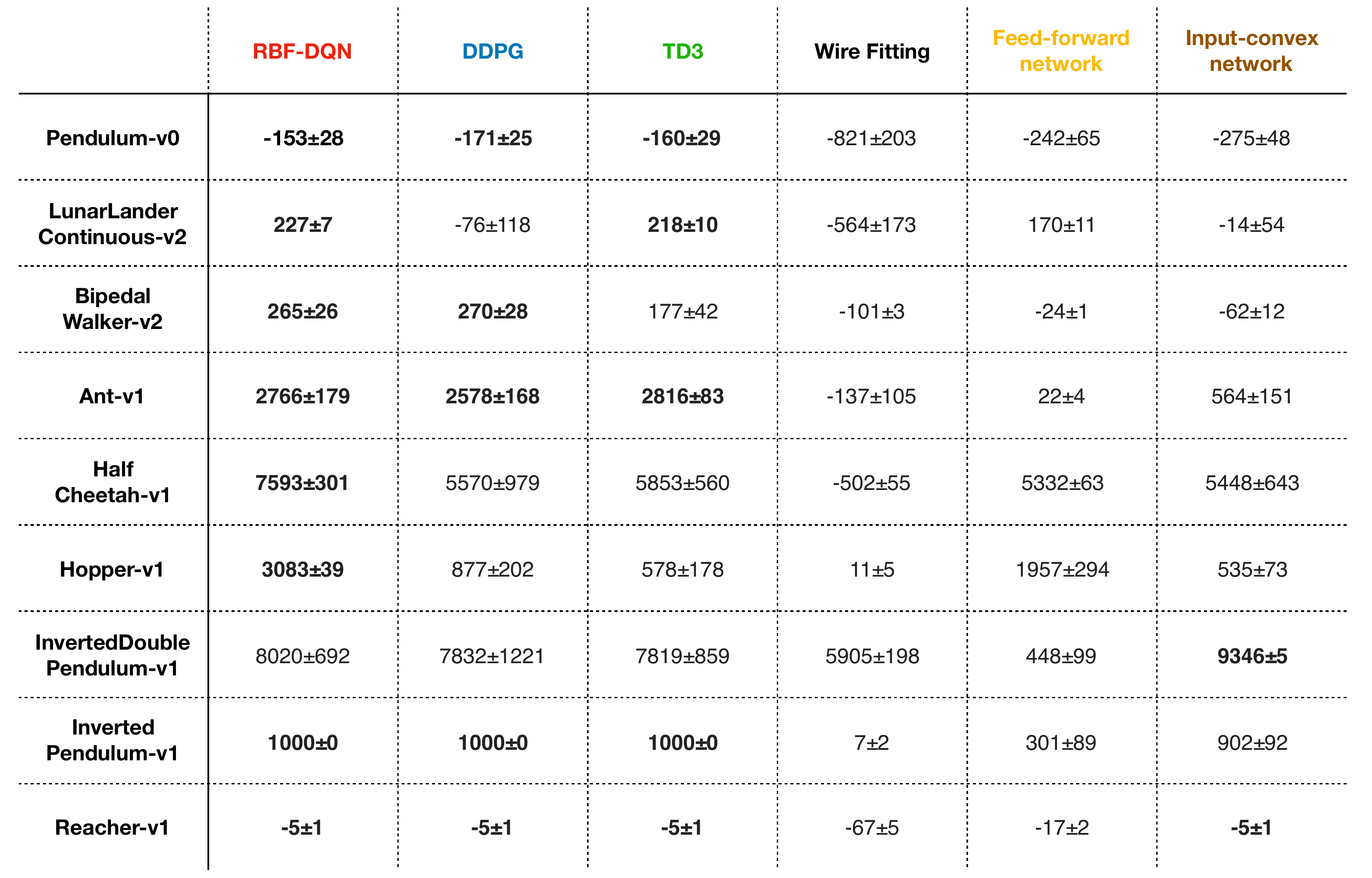}
    \caption{A comparison between RBF-DQN and different deep-RL baselines based on final performance.}
    \label{fig:RL_final_results}
\end{figure*}
\bibliographystyle{icml2020}
\bibliography{bibfile}